\documentclass{article}

\usepackage[accepted]{icml2018}

\usepackage[subtle, wordspacing=normal, tracking=normal, bibnotes, charwidths=normal, leading, lists, paragraphs=normal, mathspacing=normal, bibliography=normal]{savetrees}
\usepackage{times}

\usepackage{microtype}
\usepackage{graphicx}
\usepackage{booktabs} % for professional tables
\usepackage{tabulary} 
\usepackage{subcaption}
\usepackage{amsmath,amssymb}

\usepackage{hyperref}

\icmltitlerunning{Finding Influential Training Samples for Gradient Boosted Decision Trees}

\usepackage{amsmath, amssymb}
\usepackage{todonotes}
\usepackage{footnote}
\usepackage{color}
\usepackage{paralist}
\usepackage{multirow}
\usepackage{amsthm}

\newcommand{\bw}{\mathbf{w}}
\newcommand{\bx}{\mathbf{x}}
\newcommand{\bX}{\mathbf{X}}
\newcommand{\objpath}[1]{P(#1)}
\newcommand{\appr}[2]{A_{#2}^{#1}}
\newcommand{\apprw}[2]{A_{#2}^{#1}(\bw)}
\newcommand{\apprh}[2]{\hat{A}_{#2}^{#1}}
\newcommand{\ba}[1]{\mathbf{A^{#1}}}
\newcommand{\da}[3]{J(\ba{#1})_{#2#3}}
\newcommand{\baw}[1]{\mathbf{A^{#1}}(\bw)}
\newcommand{\leafval}[2]{f_{#2}^{#1}}
\newcommand{\leafvalh}[2]{\hat{f}_{#2}^{#1}}
\newcommand{\leafvala}[3]{f_{#2}^{#1}(#3)}
\newcommand{\leafvalawg}[2]{f_{G;#2}^{#1}(\baw{#1-1}, \bw)}
\newcommand{\leafvalawh}[2]{f_{H;#2}^{#1}(\baw{#1-1}, \bw)}

\newcommand{\leafidx}[2]{I_{#2}^{#1}}
\newcommand{\hleafidx}[2]{\hat{I}_{#2}^{#1}}
\newcommand{\updateset}[1]{U^{#1}}
\newcommand{\delt}[2]{\Delta_{#2}^{#1}}
\newcommand{\dloss}[2]{\frac{\partial L(#2, #1)}{\partial #1}}
\newcommand{\hloss}[2]{\frac{\partial^2 L(#2, #1)}{\partial #1^2}}
\newcommand{\dlossat}[3]{\dloss{#1}{#2}\vert_{#1=#3}}
\newcommand{\hlossat}[3]{\hloss{#1}{#2}\vert_{#1=#3}}
\newcommand{\der}[2]{g^{#1}_{#2}(\appr{#1-1}{#2})}
\newcommand{\derw}[2]{g^{#1}_{#2}(\apprw{#1-1}{#2})}
\newcommand{\dernoa}[2]{g^{#1}_{#2}}
\newcommand{\hess}[2]{h^{#1}_{#2}(\appr{#1-1}{#2})}
\newcommand{\hessw}[2]{h^{#1}_{#2}(\apprw{#1-1}{#2})}
\newcommand{\hessnoa}[2]{h^{#1}_{#2}}
\newcommand{\third}[2]{k^{#1}_{#2}(\appr{#1-1}{#2})}
\newcommand{\thirdw}[2]{k^{#1}_{#2}(\apprw{#1-1}{#2})}
\newcommand{\thirdnoa}[2]{k^{#1}_{#2}}

\newcommand{\Der}[2]{G^{#1}_{#2}(\ba{#1-1})}
\newcommand{\Derw}[2]{G^{#1}_{#2}(\baw{#1-1}, \bw)}
\newcommand{\Hess}[2]{H^{#1}_{#2}(\ba{#1-1})}
\newcommand{\Hessw}[2]{H^{#1}_{#2}(\baw{#1-1}, \bw)}
\newcommand{\f}[1]{F(#1)}
\newcommand{\fh}[2]{\hat{F}_{\backslash #2}(#1)}
\newcommand{\leafrefit}{LeafRefit}
\newcommand{\leafrefiti}{\textit{\leafrefit}}
\newcommand{\fleafrefit}{FastLeafRefit}
\newcommand{\fleafrefiti}{\textit{\fleafrefit}}
\newcommand{\leafgrad}{LeafInfluence}
\newcommand{\leafgradi}{\textit{\leafgrad}}
\newcommand{\fleafgrad}{FastLeafInfluence}
\newcommand{\fleafgradi}{\textit{\fleafgrad}}

\newcommand{\Influencegrad}{\mathit{Inf_{grad}}}

\newtheorem{assumption}{Assumption}
\newtheorem{proposition}{Proposition}
\newtheorem{proxy}{Proxy}

\begin{document}
	
	\twocolumn[
	\icmltitle{Finding Influential Training Samples for Gradient Boosted Decision Trees}
	
	
	
	\icmlsetsymbol{equal}{*}
	
	\begin{icmlauthorlist}
		\icmlauthor{Boris Sharchilev}{uva,ya}
		\icmlauthor{Yury Ustinovsky}{princeton}
		\icmlauthor{Pavel Serdyukov}{ya}
		\icmlauthor{Maarten de Rijke}{uva}
	\end{icmlauthorlist}
	
	\icmlaffiliation{uva}{Informatics Institute, University of Amsterdam, Amsterdam, The Netherlands}
	\icmlaffiliation{ya}{Yandex, Moscow, Russia}
	\icmlaffiliation{princeton}{Department of Mathematics, Princeton University, Princeton, NJ, USA}

	\icmlcorrespondingauthor{Boris Sharchilev}{bshar@yandex-team.ru}
	\icmlcorrespondingauthor{Pavel Serdyukov}{pavser@yandex-team.ru}
	\icmlcorrespondingauthor{Maarten de Rijke}{derijke@uva.nl}
	
	\icmlkeywords{Tree Ensembles, Gradient Boosting, Influence, Interpretations}
	
	\vskip 0.3in
	]
	
	
	
	\printAffiliationsAndNotice{}  % leave blank if no need to mention equal contribution

	% !TEX root = ./main.tex

\begin{abstract}
	We address the problem of finding influential training samples for a particular case of tree ensemble-based models, e.g., Random Forest (RF) or Gradient Boosted Decision Trees (GBDT). 
	A natural way of formalizing this problem is studying how the model's predictions change upon leave-one-out retraining, leaving out each individual training sample. 
	Recent work has shown that, for parametric models, this analysis can be conducted in a computationally efficient way. 
	We propose several ways of extending this framework to non-parametric GBDT ensembles under the assumption that tree structures remain fixed. 
	Furthermore, we introduce a general scheme of obtaining further approximations to our method that balance the trade-off between performance and computational complexity. 
	We evaluate our approaches on various experimental setups and use-case scenarios and demonstrate both the quality of our approach to finding influential training samples in comparison to the baselines and its computational efficiency.\footnote{Supporting code for the paper is available at \url{https://github.com/bsharchilev/influence_boosting}.}
\end{abstract}
	% !TEX root = ./main.tex

\section{Introduction and Background}
\label{section-introduction}
	As machine learning-based models become more widespread and grow in both scale and complexity, methods of interpreting their predictions are increasingly attracting attention from the machine learning community. 
	Some of the applications and benefits of employing these methods outlined in previous work~\citep{ancona-unified-2017} include \begin{inparaenum}[(1)] \item ``debugging'' the model to expose ways of model failures not discoverable via conventional test set performance measuring (e.g., data or target leakages); \item boosting developer's trust in the model's performance in scenarios when on-line evaluation is not available before deployment; and \item increasing user satisfaction and/or confidence in provided predictions, etc\end{inparaenum}. Various problem setups~\citep{palczewska-interpreting-2013,tolomei-interpretable-2017,fong-interpretable-2017} and interpretation methods, both model-agnostic~\citep{ribeiro-should-2016,lundberg-consistent-2017} and model-specific~\cite{shrikumar-learning-2017,tolomei-interpretable-2017,sundararajan-axiomatic-2017}, have recently been proposed in the literature.
	
	A common trait shared by the majority of these methods is that they treat the provided model as a \emph{fixed} function of input objects and study which features had the largest effect on the prediction, how the model responds to  feature perturbations, etc. 
	However useful they are, the obtained interpretations do not provide a way of automatically \emph{improving} the model, since the model is fixed; the main use-case thus becomes manual analytics by the user or the developer, which is both time and resource-consuming. 
	It is thus desirable to derive a framework for obtaining \emph{actionable} insights into the model's behavior allowing us to automatically improve a model's performance.
	
	One such framework has recently been introduced by \citet{koh-2017-understanding}; it deals with finding the most influential training objects. 
	They formalize the notion of ``influence" via an infinitesimal approximation to leave-one-out retraining: the core question that this work aims to answer is ``how would the model's performance on a test object $\bx_{test}$ change if the weight of a training object $\bx_{train}$ is perturbed?" 
	Assuming a smooth parametric model family (e.g., linear models or neural networks), the authors employ the Influence Functions framework from classical statistics~\citep{cook-characterizations-1980} to show that this quantity can be estimated much faster than via straightforward model retraining, which makes their method tractable in a real-world scenario. 
	A natural use-case of such a framework is to consider individual test objects (or groups of them) on which the model performs poorly and either remove the most ``harmful" training objects or prioritize a batch of new objects for labeling based on which ones are expected to be the most ``helpful," akin to active learning.
	
	Unfortunately, the method suggested by \citet{koh-2017-understanding} heavily relies on the smooth parametric nature of the model family. 
	While this is a large class of machine learning models, it is by far not the only one.
	In particular, decision tree ensembles such as Random Forests~\citep[][RF]{ho-random-1995} and Gradient Boosted Decision Trees~\citep[][GBDT]{friedman-greedy-2001} are probably the most widely used model family in industry, largely due to their state-of-the-art performance on structured and/or multimodal data. 
	Thus, it is important to extend the aforementioned Influence Functions framework to tree ensembles.
	
	In this paper, we propose a way of doing so, while focusing specifically on GBDT.
	We consider two \emph{proxy} metrics for the informal notion of influence. For the first one, leave-one-out retraining, we utilize the inner mechanics of fitting decision trees (in particular, assuming that a small training sample perturbation does not change the trees' structures) to derive \leafrefiti~and~\fleafrefiti, a well-founded family of approximations to leave-one-out retraining that trade off approximation accuracy for computational complexity.
	For the second, analogously to the Influence Functions framework, we consider infinitesimal training sample weight perturbations and derive~\leafgradi~and~\fleafgradi, methods for estimating gradients of the model's predictions with respect to training objects' weights. 
	From a theoretical perspective, \leafgradi{} and~\fleafgradi{} allow us to deal with the discontinuous dependency of tree structure on training sample perturbations; from a practical one, they allow us to further reduce computational complexity due to the possibility of precomputing certain derivatives.
	
	In our experiments we \begin{inparaenum}[(1)] \item study the conditions under which our methods, \fleafrefiti~and~\fleafgradi, successfully approximate their proxy metrics, \item demonstrate our methods' ability to target training objects which are influential for specific test objects, and \item show that our algorithms run much faster than straightforward retraining, which makes them applicable in practical scenarios.
	\end{inparaenum}
	
	% !TEX root = ./main.tex

\begin{table}[t]
	\caption{Mathematical notations used in the paper.}
	\label{notations}
	\begin{small}
		\begin{tabulary}{\columnwidth}{LL}
			\toprule
			Notation & Description \\
			\midrule
			$\bx=(x, y)$ & Data point \\
			$\bX=\{(x_i, y_i)\}_{i=1}^n$ & Training/test sample \\
			$L(y_{true},y_{pred})$ & Loss function \\
			$\bw=(w_1,...,w_n)$ & Weights of training samples \\
			$P(x)=(i_1,...,i_T)$ & Path (leaf indices) of $x$ \\
			$F(x)$ & GBDT prediction at point $x$ \\
			$f_l^t $ & Value in leaf $l$ at step $t$ \\
			$I_l^t$ & Training points belonging to leaf $l$ at step $t$ \\
			$\ba{t}=\{A_i^t:=\sum_{\tau=1}^t\leafval{\tau}{P(x_i)_\tau}\}_{i=1}^n$ & Intermediate predictions on $\{x_i\}_{i=1}^n$ \\
			$\der{t}{i}:=\dlossat{z}{y_i}{\appr{t-1}{i}}$ & $i$-th first derivative at training step $t$ \\
			$\hess{t}{i}:=\hlossat{z}{y_i}{\appr{t-1}{i}}$ & $i$-th second derivative at training step $t$ \\
			$k^t_i(A^{t-1}_i):=\frac{\partial^3 L(y_i, z)}{\partial z^3}\vert_{z=\appr{t-1}{i}}$ & $i$-th third derivative at training step $t$ \\
			$\Der{t}{l}:=\sum_{j\in\leafidx{t}{l}}w_j\der{t}{j}$ & Sum of leaf derivatives \\
			$\Hess{t}{H;l}:=\sum_{j\in\leafidx{t}{l}}w_j\hess{t}{j}$ & Sum of leaf second derivatives \\
			$\Hess{t}{G;l}:=\sum_{j\in\leafidx{t}{l}}w_j$ & Sum of leaf weights \\
			$\Influencegrad(\bx_1, \bx_2)$ & Influence of object $\bx_1$ on $\bx_2$ \\
			\bottomrule
		\end{tabulary}
	\end{small}
\vskip -0.8cm
\end{table}

\section{Problem Definition}
\label{section-problem-definition}
\label{section-approach-setup}

	First, we formally define the problem setup. We consider standard supervised training of a GBDT ensemble\footnote{Mathematical notations are defined in Table~\ref{notations}.} $F(x;\bw):=\sum_{t=1}^T\leafvala{t}{\objpath{x}_t}{\ba{t-1}}$ on a training sample $\bX_{train}$. 
	Learning consists of two separate stages: \emph{model structure selection} and \emph{picking the optimal leaf values}. 
	The way of choosing the model structure is not important for our work; we refer the interested reader to existing implementations, e.g., \citet{chen-2016-xgboost,dorogush-2017-fighting}. For picking optimal leaf values, we consider two most commonly used formulas:

	\noindent
	\textbf{Gradient}: At leaf $l$ at step $t$, output negative average gradients (calculated at current predictions) over the leaf objects:
		\begin{equation}
		\begin{gathered}
			\leafvala{t}{G;l}{\ba{t-1}}:=-\frac{\Der{t}{l}}{\Hess{t}{G;l}}.
		\end{gathered}
		\label{leaf-formula-gradient}
		\end{equation}
	This is equivalent to minimizing the empirical loss function w.r.t.\ the current leaf value by doing a single gradient step in function space~\citep{chen-2016-xgboost}.
	
	\noindent
	 \textbf{Newton}: At leaf $l$ at step $t$, output the negative total gradient divided by the total second derivative over the leaf objects:
		\begin{equation}
		\begin{gathered}
			\leafvala{t}{H;l}{\ba{t-1}}:=-\frac{\Der{t}{l}}{\Hess{t}{H;l}}.
		\end{gathered}
		\label{leaf-formula-newton}
		\end{equation}
	This is equivalent to minimizing the empirical loss function w.r.t.\ the current leaf value by doing a single Newton step in function space~\citep{chen-2016-xgboost}.

%%%%%%%%%%%%%

\section{Approach}
\label{section-approach}

In this section, we describe our approach to efficiently calculating the influence of training points.
Since the notion of ``influence" is not rigorously defined and partly intuitive, we need to introduce a well-defined, measurable quantity that aims to capture the desired intuition; we refer to it as a \emph{proxy} for influence.
In this work, we follow the general framework of~\citet{koh-2017-understanding} and quantify influence through train set perturbations. We consider two proxies that reflect two natural variations of this approach.
First, we describe an algorithm for faster exact leave-one-out retraining of GBDT under the assumption that the model structure remains fixed, and explain how to use that framework for estimating the influence of training points on specific test samples; we then introduce a general approach to obtaining approximations to this scheme for increased computational efficiency. 
Finally, we derive an iterative algorithm to compute gradients of GBDT predictions w.r.t. the weights of training sample and analyze the resulting expressions.

\subsection{Leave-One-Out Retraining}
\label{section-approach-loo}
	For the first proxy, following~\citet{koh-2017-understanding}, we quantify the (negative) \emph{influence} of a training sample $\bx_{train}$ on a model's prediction on a test sample $F(x_{test};\bw)$ as the change of loss on $\bx_{test}$ after retraining the model without $\bx_{train}$:
	\begin{proxy}
	$\Influencegrad(\bx_{train}, \bx_{test}):=L(y_{test},\f{x_{test}})- L(y_{test},\fh{x_{test}}{x_{train}}),$
	where $\hat{F}_{\backslash x_{train}}$ is the model retrained without $x_{train}$.
	\label{proxy:loo}
	\end{proxy}
	
	Since, in order to rank the training points according to $\Influencegrad(\bx_{train}, \bx_{test})$, we would have to compute Proxy~\ref{proxy:loo} for each $\bx_{train}$,  straightforward leave-one-out retraining would be prohibitively expensive even for moderately-sized datasets. 
	Moreover, as mentioned in Section~\ref{section-introduction}, the parametric model framework of \citet{koh-2017-understanding} is not directly applicable here. 
	Thus, a solution tailored specifically for tree ensembles is required.
	
	\subsubsection{\leafrefit}
	\label{section-approach-loo-leafrefit}
		In the problem definition (Section~\ref{section-approach-setup}) we noted that training each tree requires picking its structure and leaf values. 
		Moreover, these two operations respond to small training set perturbations differently: the tree structure is piecewise constant (i.e., it either stays the same or changes abruptly), whereas leaf values change more smoothly. 
		Thus, a natural assumption to make is:
		\noindent
		\begin{assumption}
			The effect of removing a single training point can be estimated while treating each tree's structure as fixed. 
			\label{tree-structure-fixed-assumption}
		\end{assumption}
		Under Assumption~\ref{tree-structure-fixed-assumption}, it is thus sufficient to estimate how the leaf values of each tree are going to change. Since selecting optimal feature splits, e.g. via CART~\citep{quinlan-induction-1986} or C4.5~\citep{quinlan-c4-2014} algorithms, is often the computational bottleneck in fitting decision trees, this observation already yields a significant complexity reduction.
		
		Thus, our first algorithm for approximate leave-one-out retraining, \leafrefiti, is equivalent to fixing the structure of every tree and fitting leaf values without the removed point. A formal listing of the resulting algorithm is given in Algorithm~\ref{alg:leafrefit}.
		
		\begin{algorithm}[tb]
			\begin{small}
			   \caption{\leafrefiti}
			\label{alg:leafrefit}
			\begin{algorithmic}[1]
				\STATE {\bfseries Input:} training point index to remove $i_0$, sample-to-leaf assignments $\{\leafidx{t}{l}\}_{t=1,l=1}^{T, L}$, leaf formula type $\mathit{formula}$
				\STATE{\bfseries Output:} new leaf values $\{\leafvalh{t}{l}\}_{t=1, l=1}^{T, L}$
				\STATE Initialize $\delt{0}{i}\gets 0, \appr{0}{i}\gets\,0,\,\,i=1\dots~n$
				\FOR{$t=1$ \textbf{to} $T$}
					\STATE $\apprh{t-1}{i}\gets\appr{t-1}{i}+\delt{t-1}{i},\,\,i=1\dots~n$
					\FOR{$l=1$ \textbf{to} $L$}
						\STATE $\hleafidx{t}{l}\gets\leafidx{t}{l}\setminus\{i_0\}$ \label{alg:leafrefit:naive}
						\IF{$\mathit{formula}==\text{Gradient}$}
							\STATE $\leafvalh{t}{l}\gets\leafvala{G;t}{l}{\{\apprh{t-1}{i}\}_{i\in\hleafidx{t}{l}}}$ 
						\ELSE
							\STATE $\leafvalh{t}{l}\gets\leafvala{N;t}{l}{\{\apprh{t-1}{i}\}_{i\in\hleafidx{t}{l}}}$
						\ENDIF
						\STATE $\Delta\leafval{t}{l}\gets\leafvalh{t}{l}-\leafval{t}{l}$
						\STATE $\delt{t}{i}\gets\delt{t-1}{i}+\Delta\leafval{t}{l}, i\in\leafidx{t}{l}$ \label{alg:leafrefit:detail}
					\ENDFOR
				\ENDFOR
				\STATE \textbf{return} $\{\leafvalh{t}{l}\}_{t=1,l=1}^{T, L}$
			\end{algorithmic}
		\end{small}
		\end{algorithm}
	
Note that the effect of removing a training object $\bx_i$ is twofold: on each step, we have to \begin{inparaenum}[(1)] \item remove $\bx_i$ from its leaf (Algorithm~\ref{alg:leafrefit}, line~7) and \item recalculate the leaf values and record the resulting changes of intermediate predictions for each training object (line~14)\end{inparaenum}. Thus, despite improving upon straightforward retraining by not having to search for the optimal tree splits, \leafrefiti{} is still an expensive algorithm. Running it for each training sample has an asymptotic complexity of $O(Tn^2)$; moreover, in practice, for each training step $t$ it involves an expensive routine of recalculating derivatives for each training point.

	\subsubsection{\fleafrefit}
	\label{section-approach-loo-fastleafrefit}
		 We seek to limit the number of calculations at each step of \leafrefiti{}.		
		Note that, in \leafrefiti, we generally cannot make any use of caching the original first and/or second derivatives, since any $\delt{t-1}{i}$ (Algorithm~\ref{alg:leafrefit}, line~14) can be nonzero, which forces us to recompute the derivatives for each object.
		We build on the intuition that, in practice, a lot of $\delt{t-1}{i}$ may be negligible; an extreme example is when training samples can be separated in disjoint cliques, i.e., $\leafidx{t_1}{l} = \leafidx{t_2}{l}\,\,\forall t_1,t_2=1,\ldots,T$, $l=1\ldots L$. 
		In this case, removing each training point only affects its clique $I_{l_0}:=\leafidx{1}{l_0}$, since objects not sharing leaves with $i$ will not be affected: $\delt{t-1}{i}=0\,\,\forall t=1...T, i\notin I_{l_0}$. 
		Thus, at each training step $t$, we may select a subset of training samples\footnote{Methods of selecting $\updateset{t}$ will be given below.} $\updateset{t}$ whose deltas we take into account, and suppose $\apprh{t-1}{i}=\appr{t-1}{i}\,\,\forall i\notin\updateset{t}$. 
		We refer to $\updateset{t}$ as the \emph{update set}. 
		Combining this with caching the original $\appr{t-1}{i}$ and sums of derivatives in each leaf, we reduce the asymptotic complexity to $O(TnC)$, where $C=\max_t\vert\updateset{t}\vert$, which is a significant reduction if $C\ll n$. 
		A formal listing of the resulting algorithm, \fleafrefiti, is given in Algorithm~\ref{alg:fleafrefit}.
		
		\begin{algorithm}[tb]
			\begin{small}
			\caption{\fleafrefiti}
			\label{alg:fleafrefit}
		\begin{algorithmic}
			\STATE {\bfseries Input:} $i_0$, $\{\leafidx{t}{l}\}_{t=1,l=1}^{T, L}$, $\{\der{t}{i}\}_{t=1,i=1}^{T,n}$, $\{\hess{t}{i}\}_{t=1,i=1}^{T,n}$, $\{\Der{t}{l}\}_{t=1,l=1}^{T,L}$, $\{\Hess{t}{l}\}_{t=1,l=1}^{T,L}$, leaf formula type $\mathit{formula}$
			\STATE \textbf{Output:} New leaf values $\{\leafvalh{t}{l}\}_{t=1,l=1}^{T,L}$
			\STATE Initialize $\delt{0}{i}\gets 0, i=1...n$
			\FOR{$t=1$ \textbf{to} $T$}
				\STATE $\updateset{t}\gets\text{UpdateSet(t)}$
				\FOR{$l=1$ \textbf{to} $L$}
					\STATE $\updateset{t}_l\gets\updateset{t}\cap\leafidx{t}{l}$
					\STATE $\leafvalh{t}{l}\gets$ \textit{LeafRecalc}$(t, l, \{\leafidx{t}{l}\}_{t=1,l=1}^{T, L}, \{\der{t}{i}\}_{t=1,i=1}^{T,n},$\\ $\{\hess{t}{i}\}_{t=1,i=1}^{T,n}, \Der{t}{l}, \Hess{t}{l}, \updateset{t}_l, \mathit{formula})$
					\STATE $\Delta\leafval{t}{l}\gets\leafvalh{t}{l}-\leafval{t}{l}$
					\STATE $\delt{t}{i}\gets\delt{t-1}{i}+\Delta\leafval{t}{l}, i\in\leafidx{t}{l}$
				\ENDFOR
			\ENDFOR
			\STATE \textbf{return} $\{\leafvalh{t}{l}\}_{t=1,l=1}^{T,L}$
		\end{algorithmic}
		\end{small}
		\end{algorithm}
		
		\begin{algorithm}[tb]
			\caption{\textit{LeafRecalc}}
			\label{alg:leafrecalc}
		\begin{algorithmic}
			\STATE \textbf{Input:} boosting step $t$, leaf index $l$, $\{\leafidx{t}{l}\}_{t=1,l=1}^{T, L}$, $\{\der{t}{i}\}_{t=1,i=1}^{T,n}$, $\{\hess{t}{i}\}_{t=1,i=1}^{T,n}$, $\Der{t}{l}$, $\Hess{t}{l}$, $\updateset{t}_l$, leaf formula type $formula$
			\STATE \textbf{Output:} New leaf value $\leafvalh{t}{l}$
			\STATE $I\gets I[i_0\in\leafidx{t}{l}]$
			\STATE $\Delta\,g^t_j\gets g^t_j(\appr{t-1}{j}+\delt{t-1}{j})-g^t_j(\appr{t-1}{j}), j\in\updateset{t}_l$
			\STATE $\hat{G}^t_l\gets\Der{t}{l}+\sum_{j\in\updateset{t}_l}w_j\Delta\,g^t_j-Iw_{i_0}\der{t}{i_0}$
			\IF{$formula==\text{Gradient}$}
				\STATE $\hat{H}^t_l\gets\sum_{j\in\leafidx{t}{l}\setminus\{i_0\}}w_j$
			\ELSE
				\STATE 	$\Delta\,h^t_j\gets h^t_j(\appr{t-1}{j}+\delt{t-1}{j})-h^t_j(\appr{t-1}{j}), j\in\updateset{t}_l$
				\STATE $\hat{H}^t_l\gets\Hess{t}{l}+\sum_{j\in\updateset{t}_l}w_j\Delta\,h^t_j-Iw_{i_0}\hess{t}{i_0}$
			\ENDIF
			\STATE \textbf{return} $-\frac{\hat{G}^t_l}{\hat{H}^t_l}$
		\end{algorithmic}
		\end{algorithm}
	
	\subsubsection{Selecting the update set}
	\label{section-approach-loo-updateset}
		In Section~\ref{section-approach-loo-fastleafrefit}, we introduced \fleafrefiti, an approximate algorithm
		potentially achieving lower complexity than \leafrefiti. 
		Its definition, however, allowed for an arbitrary choice of the \emph{update set} $\updateset{t}$ telling us which training points' prediction changes to take into account at boosting step $t$. 
		It is intuitively clear that different strategies of selecting $\updateset{t}$ allow us to optimize the trade-off between computational complexity and quality of approximating leave-one-out retraining; thus, \fleafrefiti~provides a principled way of obtaining approximations of different rigor to \leafrefiti.		
		Natural strategies for selecting the update set include:
		
		\noindent
		\textbf{SinglePoint}: don't update any points' predictions and only ignore the derivatives of  $i$ (the index of the training point to be removed) in each leaf, i.e., $\updateset{t}=\emptyset$. 
		Also note that this strategy is equivalent to disregarding dependencies between consecutive trees in GBDT and treating the ensemble like a Random Forest. 
		Its complexity is $O(Tn)$.
		
		\noindent
		\textbf{AllPoints}: make no approximations and update each point at each step, i.e., $\updateset{t}=\{1,\ldots,\vert\bX_{train}\vert\}$. This reduces \fleafrefiti~to~\leafrefiti.
		
		\noindent
		\textbf{TopKLeaves(k)}: this heuristic builds on the observation that, at each step $t$, each $\delt{t}{j}, j\in\leafidx{t}{l}$ increases over $\delt{t-1}{j}$ by the same amount $\Delta f^{t}_{l}$ across the leaf $l$ (see Algorithm~\ref{alg:fleafrefit}). 
		$\Delta f^{t}_{l}$'s magnitude, in turn, is expected to be larger for leaves where $\delt{t-1}{j}, j\in\leafidx{t}{l}$ (and, subsequently, $\Delta g^t_j$) are already large.
		Informally, the ``snowball" effect holds: the larger the change accumulated in the leaf so far, the greater its value will change. 
		Thus, to exploit this intuition, $\mathit{TopKLeaves}(k)$ only updates $\delt{t}{j}$ of training points  in $k$ leaves with the largest accumulated prediction change so far:
		\noindent
			\begin{equation}
			\begin{gathered}
				\updateset{t}=\{i\in\leafidx{t}{l}\mid l\in\{L^t_j\}_{j=1}^k\},\\
				L^t=\text{argsort}\bigl[-\sum\nolimits_{i\in\leafidx{t}{l}}\vert\delt{t-1}{i}\vert, l=1\dots L\bigr]
			\end{gathered}
			\label{eqn:topkleaves-definition}
			\end{equation}
		\textbf{Note}: despite the speedup from omitting unimportant leaves, this strategy is formally still $O(Tn^2)$ due to the fact that computing $\updateset{t}$ according to Eq.~\ref{eqn:topkleaves-definition} takes $O(n)$.
		In practice, overhead for computing Eq.~\ref{eqn:topkleaves-definition} may be negligible because, firstly, sums of $\delt{t-1}{i}$ can be quickly computed in a parallel or vectorized fashion and, secondly, because the complexity of addition is negligible compared to, e.g., calculating derivatives. 
		However, if this still poses a problem, a natural way of getting around it is sampling $m$ training points uniformly from $\bX_{train}$ and using a sample estimator of Eq.~\ref{eqn:topkleaves-definition}. The complexity of~\fleafrefiti~thus becomes $O(Tn[C+m])$, which is useful if $m\ll n$.
	
	\subsection{Prediction gradients}
	\label{section-approach-gradients}
		In the previous sections we introduced \leafrefiti~and~\fleafrefiti, fast methods of estimating the effect of a training sample on the GBDT ensemble, which can then be used to rank training points, e.g., by their influence on a test point of interest. 
		Under Assumption~\ref{tree-structure-fixed-assumption}, these methods are valid approximations of leave-one-out retraining, which gives them theoretical grounding. 
		However, as shown in Section~\ref{section-evaluation-proxy}, when Assumption~\ref{tree-structure-fixed-assumption} is violated, \leafrefiti~and~\fleafrefiti~are no longer valid approximations to Proxy~\ref{proxy:loo}.
	
		The intuition underlying Assumption~\ref{tree-structure-fixed-assumption}, however, still holds: for a small enough perturbation to the training data, the structure will remain fixed, whereas leaf values will still be changing smoothly. 
		Note that retraining the model without a sample $i$ is equivalent to setting $w^{new}_{i}=w^{old}_{i}+\Delta w_{i}; \Delta w_{i}=-w^{old}_{i}$. 
		This change may be large enough to trigger structural shifts in the ensemble; thus, we need a tool to study a model's response to smaller (arbitrarily small) perturbations.
	
		An obvious choice for such a tool is the derivative of a model's prediction w.r.t.\ a sample's weight, which was also a crucial tool in the Influence Functions framework from classical statistics~\citep{cook-characterizations-1980}:
		\begin{proxy}
		$\Influencegrad(\bx_{train}, \bx_{test}):=\frac{\partial L(y_{test}, F(x_{test}))}{\partial w_{i(x_{train})}},$
		where $i(x_{train})$ is the index of $\bx_{train}$ in $\bX_{train}$.
		\label{proxy:gradients}
		\end{proxy}
	
	\subsubsection{LeafInfluence}
	\label{section-approach-gradients-leafinfluence}
	As mentioned above, in the setup of Proxy~\ref{proxy:gradients} the statement of Assumption~\ref{tree-structure-fixed-assumption} is now guaranteed to hold and is no longer an assumption; we may consider the tree structures to be fixed and only study perturbations of leaf values, which smoothly depend on the weights. 
	Using the chain rule
		\begin{gather}
		 \frac{\partial L(y, F(x; \bw))}{\partial w_i}=\frac{\partial L(y, z)}{\partial z}\vert_{z=F(x; \bw)}\cdot\frac{\partial F(x;\bw)}{\partial w_i},
		 \label{eqn:chain-rule}
		 \end{gather}
		 we can then derive various counterfactuals (e.g., ``how would the loss on a test point change if we upweight a training point $i$?"), similarly to \citet{koh-2017-understanding}.		
		Since we have
		\begin{gather}
			\frac{\partial F(x;\bw)}{\partial w_{i}}=\sum_{t=1}^T\frac{\partial \leafvala{t}{\objpath{x}_t}{\ba{t-1}}}{\partial w_i},
			\label{eqn:prediction-to-sum-leaf-ders}
		\end{gather}
		for applying Eq.~\ref{eqn:chain-rule} to arbitrary $x$ (for a fixed $i$) it is necessary and sufficient to calculate $\frac{\partial \leafvala{t}{l}{\ba{t-1}}}{\partial w_i},\,\,t=1\ldots T$, $l=1\ldots L$. Applying Eq.~\ref{eqn:chain-rule} can then be done by running $x$ though a new tree ensemble having $\{\frac{\partial \leafvala{t}{l}{\ba{t-1}}}{\partial w_i}\}_{t=1,l=1}^{T,L}$ as leaf values.
		
		Expressions for leaf value derivatives depend on the type of leaf formula:\footnote{In Proposition~\ref{prop:leaf-derivatives}'s statement, arguments such as $\bw$ or $\appr{t}{i}$ are dropped for brevity.}
		\begin{proposition}
	        		\label{prop:leaf-derivatives}
			Leaf value derivatives are given by:
			\begin{small}
			\begin{equation}
			\begin{aligned}
				\mbox{}\!\!\!\frac{\partial \leafval{t}{G;l}}{\partial w_i}&=-\textstyle{\frac{I^{t}_{l}(i)(\leafval{t}{G;l}+\dernoa{t}{i})+\sum_{j\in\leafidx{t}{l}}w_j\hessnoa{t}{j}\da{t-1}{i}{j}}{H^t_{G;l}}} \text{ and }\\
				\mbox{}\!\!\!\frac{\partial \leafval{t}{H;l}}{\partial w_i}&=-\textstyle{\frac{I^{t}_{l}(i)(\hessnoa{t}{i}\leafval{t}{H;l}+\dernoa{t}{i})+\sum_{j\in\leafidx{t}{l}}w_j(\thirdnoa{t}{j}\leafval{t}{H;l}+\hessnoa{t}{j})\da{t-1}{i}{j}}{H^t_{H;l}}},
			\end{aligned}
			\label{eqn:leaf-derivatives}
			\end{equation}
			\end{small}
			where $\leafidx{t}{l}(i):=I[i\in\leafidx{t}{l}]$ and  $\da{t}{i}{j}:=\frac{\partial \appr{t}{j}(\bw)}{\partial w_i}$.
		\end{proposition}
		\begin{proof}
			First, let us derive the desired expression\footnote{Throughout this proof, we add $\bw$ as an extra argument to the functions we study in order to highlight the dependency.} for $\leafvalawg{t}{l}$:
			\begin{equation}
			\begin{aligned}
			&\frac{\partial \leafvalawg{t}{l}}{\partial w_i}=-\frac{\partial}{\partial w_i}\left[\frac{\Derw{t}{l}}{\Hessw{t}{G;l}}\right]=\\
			&=-\frac{\frac{\partial \Derw{t}{l}}{\partial w_i}\Hessw{t}{G;l}}{\Hessw{t}{G;l}^2}+\\
			&+\frac{\frac{\partial \Hessw{t}{G;l}}{\partial w_i}\Derw{t}{l}}{\Hessw{t}{G;l}^2}
			\end{aligned}
			\label{eqn:decomposition}
			\end{equation}
			Let us calculate the derivatives of $\Derw{t}{l}$ and $\Hessw{t}{G;l}$ separately:
			\begin{equation*}
			\begin{aligned}
			&\frac{\partial \Derw{t}{l}}{\partial w_i} =\\ &=\sum_{j\in\leafidx{t}{l}}\left[\delta_{ij}\derw{t}{j}+w_j\hessw{t}{j}\da{t-1}{i}{j}\right]=\\
			&=\leafidx{t}{l}(i)\derw{t}{i}+\sum_{j\in\leafidx{t}{l}}w_j\hessw{t}{j}\da{t-1}{i}{j};\\
			&\frac{\partial \Hessw{t}{G;l}}{\partial w_i}=\leafidx{t}{l}(i)
			\end{aligned}
			\end{equation*}
			Plugging this back into Equations~\ref{eqn:decomposition} and grouping terms with and without $\leafidx{t}{l}(i)$ separately, we get:
			\begin{equation*}
			\begin{aligned}
			&\frac{\partial \leafvalawg{t}{l}}{\partial w_i}=-\leafidx{t}{l}(i)\frac{\leafval{t}{G;l}+\derw{t}{i}}{\Hessw{t}{G;l}}-\\
			&-\frac{\sum_{j\in\leafidx{t}{l}}w_j\hessw{t}{j}\da{t-1}{i}{j}}{\Hessw{t}{G;l}},
			\end{aligned}
			\end{equation*}
			which proves the first part of Proposition~\ref{prop:leaf-derivatives}.
			
			For the second part, all we have to change is to substitute $\Hessw{t}{H;l}$ for $\Hessw{t}{G;l}$. Its derivative is given by
			\begin{equation*}
			\begin{aligned}
			&\frac{\partial \Hessw{t}{G;l}}{\partial w_i}=\\
			&=\sum_{j\in\leafidx{t}{l}}\left[\delta_{ij}\hessw{t}{j}+w_j\thirdw{t}{j}\da{t-1}{i}{j}\right]=\\
			&=\leafidx{t}{l}(i)\hessw{t}{i}+\sum_{j\in\leafidx{t}{l}}w_j\thirdw{t}{j}\da{t-1}{i}{j}
			\end{aligned}
			\end{equation*}
			Just like before, plugging it back into Equations~\ref{eqn:decomposition} and grouping terms containing and not containing $\leafidx{t}{l}(i)$ separately, we get:
			\begin{equation*}
			\begin{aligned}
			&\frac{\partial \leafvalawh{t}{l}}{\partial w_i}=-\leafidx{t}{l}(i)\frac{\hess{t}{i}\leafval{t}{H;l}+\der{t}{i}}{\Hessw{t}{H;l}}-\\
			&-\frac{\sum_{j\in\leafidx{t}{l}}w_j(\third{t}{j}\leafval{t}{H;l}+\hess{t}{j})\da{t-1}{i}{j}}{\Hessw{t}{H;l}}.
			\end{aligned}
			\end{equation*}
			This concludes the proof of Proposition~\ref{prop:leaf-derivatives}.
		\end{proof}
		It can be seen from Eq.~\ref{eqn:leaf-derivatives} that leaf value derivatives at step $t$ depend on the Jacobi matrix $\da{t-1}{i}{j}$. 
		These values, in turn, are connected by a recursive relationship:
		\begin{gather}
			\da{t}{i}{j}=\da{t-1}{i}{j}+\frac{\partial\leafval{t}{P(x_j)_t}}{\partial w_i}.
			\label{eqn:dat-recursive}
		\end{gather}
		Thus, we can calculate leaf value derivatives in an iterative fashion similar to \leafrefiti. 
		A formal listing of the resulting algorithm, \leafgradi, can be found in Algorithm~\ref{alg:leafgrad}.
		\begin{algorithm}[tb]
			\caption{\leafgradi}
			\label{alg:leafgrad}
			\begin{algorithmic}
				\STATE \textbf{Inputs:} training point index $i_0$, sample-to-leaf assignments $\{\leafidx{t}{l}\}_{t=1,l=1}^{T, L}$, $\{\der{t}{i}\}_{t=1,i=1}^{T,n}$, $\{\hess{t}{i}\}_{t=1,i=1}^{T,n}$, $\{k^t_i(A^{t-1}_i)\}_{t=1,i=1}^{T,n}$, leaf formula type $\mathit{formula}$
				\STATE \textbf{Outputs:} leaf value derivatives $\{\frac{\partial \leafvala{t}{l}{\ba{t-1}}}{\partial w_i}
			\}_{t=1,l=1}^{T,L}$
				\STATE $J(\ba{0})_{ij}\gets 0,\,\,i=1\dots n,\,\,j=1\dots n$
				\FOR{$t=1$ \textbf{to} $T$}
					\STATE $\frac{\partial \leafvala{t}{l}{\ba{t-1}}}{\partial w_{i_0}}\gets$/According to Eq.~\ref{eqn:leaf-derivatives}/, $l=1\dots L$
					\STATE $\da{t}{i}{j}\gets$/According to Eq.~\ref{eqn:dat-recursive}/,$i=1\dots n,j=1\dots n$
				\ENDFOR
				\STATE \textbf{return} $\{\frac{\partial \leafvala{t}{l}{\ba{t-1}}}{\partial w_i}
				\}_{t=1,l=1}^{T,L}$
			\end{algorithmic}
		\end{algorithm}
		Besides providing means for analyzing small weight perturbations, two important traits yielding complexity reductions can be seen from Eq.~\ref{eqn:leaf-derivatives}:
		
		\noindent
		 \textbf{A}. Using Eq.~\ref{eqn:leaf-derivatives}, we can write out $\nabla_w f^t_l=\left(\frac{\partial f^t_l}{\partial w_i}\right)_{i=1}^n$ in vector form; since computing each $\frac{\partial f^t_l}{\partial w_i}$ involves addition and a vector dot product, $\nabla_w\leafval{t}{l}$ can then be expressed via vector addition and matrix/vector product for easy parallelization/vectorization.
		 
		 \noindent
		 \textbf{B}. The derivatives $\{g^t_j, h^t_j, k^t_j\}_{t=1,j=1}^{T,n}$ used in Eq.~\ref{eqn:leaf-derivatives} can now be precomputed only once during GBDT training and not for each training object $i$ whose influence we want to compute. This contrasts~\leafgradi~with~\leafrefiti~and~\fleafrefiti, where these derivatives had to be recalculated for each $i$ depending on the values of $\delt{t-1}{j}$, which change for different~$i$.

	\subsubsection{\fleafgrad}
	\label{section-approach-loo-fleafgrad}
		The final step to be made is analogous to the transition from \leafrefiti~to~\fleafrefiti: \leafgradi~is, again, $O(Tn^2)$ because it has to compute matrix/vector products with the matrix $\da{t-1}{i}{j}$ for every $t$. 
		The same approximation that powers \fleafrefiti~can be made here as well: at each step, we can select an update set $\updateset{t}$ and only take into account the influences of a subset of training objects on $\ba{t-1}$. 
		This is equivalent to assuming $\da{t-1}{i}{j}=0\,\,\forall\,\,j\notin\updateset{t}$, making $\da{t-1}{i}{j}$ a sparse matrix with the number of nonzero elements in each row bounded by $C:=\max_t\vert\updateset{t}\vert$. 
		Strategies of selecting $\updateset{t}$ and the resulting asymptotics become the same as described in Section~\ref{section-approach-loo-updateset}, with the additional benefit of being able to compute the derivatives ``off-line."

	% !TEX root = ./main.tex

\section{Experiments}
\label{section-evaluation}
	\subsection{Research Questions}
	\label{section-evaluation-rq}
		The experiments that we conduct can be broadly categorized as serving two purposes: \begin{inparaenum}[(1)] \item studying the fundamentals of our framework and \item evaluating its quality in two applied problem setups\end{inparaenum}.
		For the first part, the research questions that we seek to answer are as follows:
		
		\noindent
		\textbf{RQ1}. How well do the different methods introduced in Sections~\ref{section-approach-loo}~and~\ref{section-approach-gradients} approximate their respective influence proxies?
		
		\noindent
		\textbf{RQ2}. Do smaller update sets significantly reduce the runtimes of~\fleafrefiti~and~\fleafgradi? Does~\fleafgradi~yield a notable runtime speedup over \fleafrefiti?

		For the second part, we proceed by considering two applied scenarios:
		\begin{inparaenum}[(1)] \item \label{scenario: 1} classification in the presence of label noise, and \item \label{scenario: 2} classification with train/test domain mismatch. \end{inparaenum} Specifically, the research questions for this part are:
		
		\noindent
		\textbf{RQ3}. For Scenario~\ref{scenario: 1}, do our methods allow to detect noise in general and, more specifically, to identify training objects most harmful for specific test points?
		
		\noindent
		\textbf{RQ4}. For Scenario~\ref{scenario: 1}, how do the proxies and their respective approximations compare in terms of quality?
		
		\noindent
		\textbf{RQ5}. For Scenario~\ref{scenario: 2}, are our methods capable of detecting domain mismatch and, moreover, providing recommendations on how to fix it?
	
	\subsection{Datasets and Framework}
	\label{section-evaluation-data}
	For our experiments with GBDT, we use CatBoost\citep{catboost} an open-source implementation of GBDT by Yandex\footnote{We use the ``Plain" mode which disables CatBoost's conceptual modifications to the standard GBDT scheme.}. 
	The datasets used for evaluation are: \begin{inparaenum}[(1)] \item Adult Data Set (\textbf{Adult}, \cite{dataset-adult}), \item Amazon Employee Access Challenge dataset (\textbf{Amazon}, \cite{dataset-amazon}), \item the KDD Cup 2009 Upselling dataset (\textbf{Upselling}, \cite{dataset-upselling}) and, for the domain mismatch experiment, \item the Hospital Readmission dataset \citep{strack-2014-impact}\end{inparaenum}. Dataset statistics and corresponding CatBoost parameters can be found in the supplementary material. Since we approach the problem as a search (for influential examples) problem, the main metrics we will be using are ranking metrics - specifically, DCG and NDCG\citep{dcg} with linear gains.

	\subsection{Proxy Approximation Quality}
	\label{section-evaluation-proxy}
		Here, we evaluate how well do variations of~\fleafrefiti~and~\fleafgradi~match their respective Proxies~\ref{proxy:loo}~and~\ref{proxy:gradients}. For that, we use the Adult Data Set.
		For \leafrefiti, its validity heavily depends on whether Assumption~\ref{tree-structure-fixed-assumption} holds; thus, we split the training points into two disjoint sets based on whether they violate Assumption~\ref{tree-structure-fixed-assumption} (\emph{Changed} in Table~\ref{table:proxy}) or not. We then randomly sample $n=2000$ points from both groups to ensure that they are equal in size and, in both of them, for each test object, we rank the train points with respect to their influence on this test object. We then measure NDCG@100 with respect to the relevance labels produced by ground-truth rankings induced by the respective proxies for \leafrefiti~and~\fleafrefiti, Proxy~\ref{proxy:loo} and Proxy~\ref{proxy:gradients}. Finally, we average the results over the test objects. Results are given in Table~\ref{table:proxy}.
				
		\begin{table}[t]
			\caption{NDCG@100 of proxy ranking approximation.}
			\label{table:proxy}
			\centering
			\begin{small}
				\begin{tabulary}{\columnwidth}{Lcccc}
					\toprule
					\multirow{2}{*}{Method} & \multicolumn{2}{c}{\fleafrefiti} & \multicolumn{2}{c}{\fleafgradi} \\
					 & Same & Changed & Same & Changed\\
					 \midrule
					SinglePoint & 0.38 & 0.10 & 0.39 & 0.80\\
					Top1Leaves & 0.41 & 0.10 & 0.43 & 0.81\\
					Top2Leaves & 0.53 & 0.10 & 0.52 & 0.83\\
					Top8Leaves & 0.87 & 0.10 & 0.87 & 0.94\\
					Top22Leaves & 0.96 & 0.10 & 0.95 & 0.98\\
					Top64Leaves & 1.00 & 0.10 & 1.00 & 1.00\\
					\bottomrule
				\end{tabulary}
			\end{small}
		\end{table}
	
	Analysis of Table~\ref{table:proxy} answers our \textbf{RQ1}. Firstly, as expected, \leafrefiti~and its faster variations only approximate Proxy~\ref{proxy:loo} when Assumption~\ref{tree-structure-fixed-assumption} holds. When it does, the quality of~\fleafrefiti~uniformly increases with the update set size, reaching perfect results for \emph{Top64Leaves}, which is equivalent to~\leafrefiti. On the other hand, \leafgradi~approximates Proxy~\ref{proxy:gradients} regardless of Assumption~\ref{tree-structure-fixed-assumption}; the dependency of~\fleafgradi~on the update set is analogous to that of~\fleafrefiti. This shows that~\leafgradi~is more robust in approximating its corresponding proxy than~\leafrefiti.

	\subsection{Runtime Comparison}
	\label{section-evaluation-runtime}
	In this section, we compare different variations (update set choices) of~\fleafrefiti~and~\fleafgradi~in terms of their runtimes. For each dataset used in the study, we randomly pick $k=100$ training objects for influence evaluation, calculate the resulting change in the model (new leaf values for~\fleafrefiti~and leaf value derivatives for~\fleafgradi), measure the total elapsed wall time and divide the result by $k$ to obtain the average time elapsed per one training object. The results are given in Fig.~\ref{fig:runtimes}.
	\begin{figure}
		\includegraphics[clip, trim=0mm 3mm 0mm 0mm, width=\linewidth]{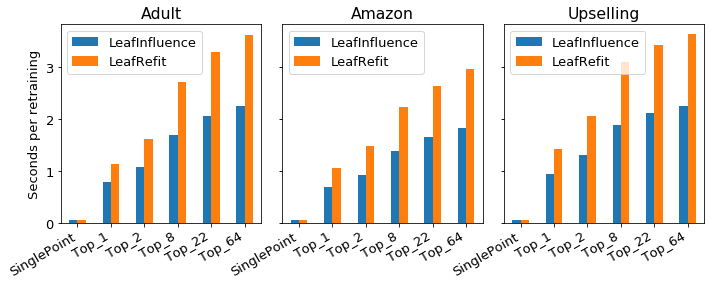}
		\caption{Wall times elapsed per training object for different variations of \fleafrefiti and \fleafgradi. \emph{Top\_k} denotes the \emph{TopKLeaves(k)} update set.}
		\label{fig:runtimes}
		\vskip -0.8cm
	\end{figure} 
	Firstly, as expected, we observe that smaller update sets considerably reduce the runtimes of our algorithms, with the most radical speedup yielded by \emph{SinglePoint} due to not having to recalculate any derivatives at all. Secondly, quite naturally, \fleafgradi~performs much faster than \fleafrefiti, presumably due to vectorization and gradient precomputation (see end of Section~\ref{section-approach-gradients-leafinfluence}). These observations confirm \textbf{RQ2}.
		
	\subsection{Harmful Object Removal}
	\label{section-evaluation-harmful}
	In this experiment, we consider a particular use-case scenario, classification in the presence of label noise, and evaluate whether our methods are able to identify training objects that are \begin{inparaenum}[(1)] \item noisy, \item harmful for specific test objects\end{inparaenum}. 
	In order to do that, we randomly select $k$ training samples,\footnote{We set $k=4000$ for Adult and Amazon, and $k=3500$ for Upselling.} flip their labels, and obtain GBDT's predictions on test data before and after noise injection. We then conduct two experiments:
	
	\begin{figure}[t]
		\centering
		\includegraphics[width=\linewidth]{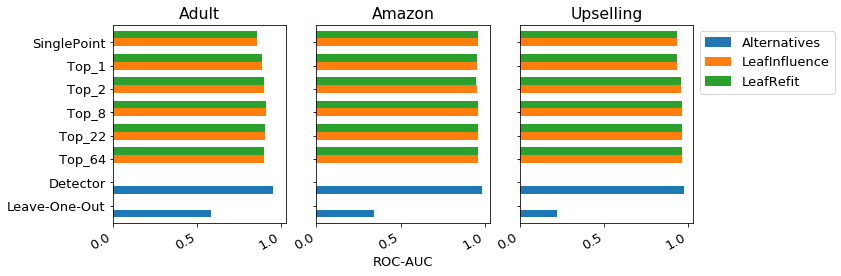}
		\caption{ROC-AUC of noise detection qualities.}
		\label{fig:noise_detection}
	\end{figure}
	
	\begin{figure}[t]
		\begin{subfigure}{\columnwidth}
			\centering
			\includegraphics[width=\linewidth]{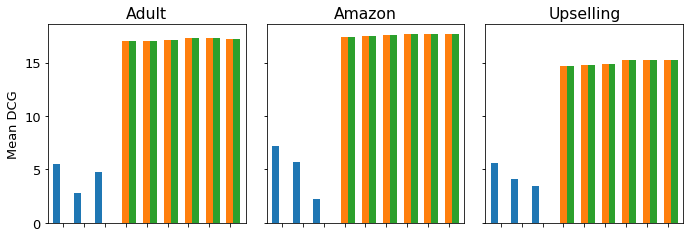}
			\caption{Logloss reduction on a particular test index.}
			\label{fig:on_idx}
		\end{subfigure}
		\begin{subfigure}{\columnwidth}
			\centering
			\includegraphics[width=\linewidth]{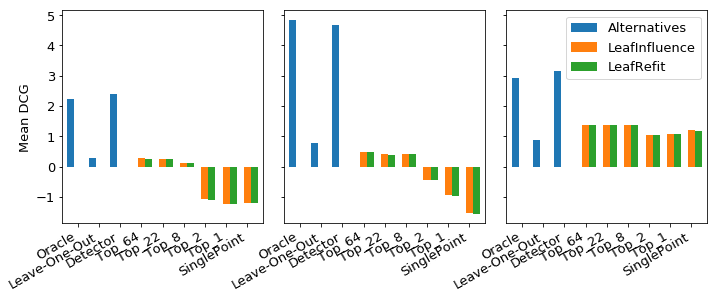}
			\caption{Logloss reduction on the whole test set.}
			\label{fig:mean_on_test}
		\end{subfigure}
		\caption{Mean DCG for relative Logloss reductions.}
		%			\subref{fig:on_idx};
		%			\subref{fig:mean_on_test} 
		\label{fig:error_fixing}
		\vskip -0.75cm
	\end{figure}
	
	\noindent
	\textbf{A}. We sort the training points in ascending order of average influence on test objects and measure ROC-AUC of noise detection. In addition to variations of~\fleafrefiti~and~\fleafgradi, we also compare against \begin{inparaenum}[(1)] \item A noise detection method exploiting the problem structure, which scores the training points using GBDT's prediction in favor of the class opposite to its observed label (\emph{Detector}), \item actual loss changes after leave-one-out retraining (\emph{Leave-One-Out}), and \item ground-truth binary labels of the train object being noisy (\emph{Oracle})\end{inparaenum}. The results are given in Fig.~\ref{fig:noise_detection}.
	
	\noindent
	\textbf{B}. We select $n=50$ test points that suffered the largest Logloss increase, thus simulating problematic test objects. For each of these objects, we sort the training points in ascending order of influence and incrementally remove them from the training set in batches of $m=50$ objects; on each iteration we measure the relative Logloss reduction both on this given test object and on the whole test set $\bX_{test}$ and, similarly to ranking, calculate DCG using these reductions as gains. Finally, we average these metrics over the $n$ test points. The results are given in Fig.~\ref{fig:on_idx}~and~\ref{fig:mean_on_test}.
	
	Firstly, from Fig.~\ref{fig:noise_detection}, we note that all variations of~\fleafrefiti~and~\fleafgradi~perform strongly on the overall noise detection problem, where they score close to the top-performing \emph{Detector}.
	Secondly, our methods greatly outperform their competitors (shown in blue on Fig.~\ref{fig:on_idx}) in targeting training objects harmful for a particular test object. These two observations confirm the hypothesis of \textbf{RQ3}.
	Finally, the two parts on Fig.~\ref{fig:error_fixing} address \textbf{RQ4} by clearly showing the way in which larger update sets increase quality: while all approximations score comparably in targeting particular test objects, smaller update sets lead to worsening the overall test quality (except for Upselling); in other words, \emph{smaller update sets lead to overfitting the targeted test object}. Proper configurations of TopKLeaves, on the other hand, allow to ``fix" a specific test object without overfitting it (k=8, 22, 64 for Adult and Amazon).

	\subsection{Debugging Domain Mismatch}
	\label{section-evaluation-readmission}
	A common issue in the supervised machine learning is \emph{domain mismatch}. This is a situation, when the joint distribution of points in the test dataset $\bX_{test}$ differs from the one in the labeled training dataset $\bX_{train}$. Often in such scenarios, a model fine-tuned on the training dataset fails to produce accurate predictions on the test data. A standard way to cope with this problem is re-weighting $\bX_{train}$.

	In the following experiment we demonstrate that by identifying influential samples in $\bX_{train}$ for certain subsamples in $\bX_{test}$ we are able to detect domain mismatch and get a hint on how the distribution of points in $\bX_{train}$ should be modified in order to match better the distribution of points in $\bX_{test}$. The design of this experiment is a modification of the corresponding use-case of~\citet{koh-2017-understanding}. We use the same Hospital dataset (see Section~\ref{section-evaluation-data}), with each point being a hospital patient represented by 127 features and the goal is to predict the readmission. To introduce domain mismatch we bias the distribution in the training dataset by filtering out a subsample of patients with $\mathrm{age}\in[40;50)$ and label $y=1$. Originally we had 169/1853 readmitted patients in this group and 2140/20000 overall; after we get 17/1601 in the $\mathrm{age}\in[40;50)$ group and 1988/19848 overall.
	Clearly, the distribution of labels in this specific age group becomes highly biased, while the proportion of positive labels in the whole dataset changes slightly (from 10.7\% to 10.0\%).
	
	Training set $\bX_{train}$ is naturally split into four parts $\{\bX_{train}^i\}_{i=1}^4$ depending on the value of $y$ and whether $\mathrm{age}\in[40;50)$.
	One would expect that in the modified training dataset, samples with  $\mathrm{age}\in[40;50)$ and $y=1$ are the most (positively) influential, so their removal will be the most harmful for the performance on the test dataset, while the removal of the samples with $\mathrm{age}\in[40;50)$ and $y=0$ might even be beneficial, since it is the most straightforward way to align the distributions in the test and train datasets.
	Below we confirm this expectation.
	
	Let us focus on the subset $\bX_{test}^0:=\{\bx\in\bX_{test}\,|\,\mathrm{age}(\bx)\in[40;50)\}$,
	since its elements are expected to be the most affected by the introduced domain mismatch. We sample 100 points from every part $\{\bX_{train}^i\}_{i=1}^4$ (or take the whole part, if it has $<100$ points). For each of the methods \emph{FastLeafRefit} and \emph{FastLeafInfluence} with various update sets we compute the influence of the training samples on $\bX_{test}^0$. Specifically, (a) with \emph{FastLeafRefit}, for an element $\bx\in\bX_{train}$ we find the average Logloss reduction on $\bX_{test}^0$, introduced by removing $\bx$; (b) with \emph{FastLeafInfluence}, for an element $\bx\in\bX_{train}$ we find the derivative of the average Logloss on $\bX_{test}^0$ with respect to the weight $w$ of $\bx$ at $w=1$.
	
	Table~\ref{table:mismatch_influence1} provides the average influence among the sampled train points with the fixed label $y\in\{0,1\}$ and the fixed indicator $I(\mathrm{age}\in[40;50))\in\{0,1\}$.
	\begin{table}[t]
		\caption{Influence of the points in $\bX_{train}$ on the loss on $\bX_{test}^0$ averaged in the corresponding sampled group (LR=\emph{LeafRefit}, LI=\emph{LeafInfluence}).}	
		\label{table:mismatch_influence1}
		\centering
		\begin{small}
			\begin{tabulary}{\linewidth}{Lcccc}
				\toprule
			& \multicolumn{2}{c}{$\mathrm{age\in[40;50)}$}
			& \multicolumn{2}{c}{$\mathrm{age\not\in[40;50)}$}\\
			 Method & $y=1$ & $y=0$ & $y=1$ & $y=0$\\
				\midrule
				LR SinglePoint & {\bf -0.525} & 0.151 & 0.084 & 0.141 \\
				LR Top1Leaves & {\bf -0.515} & 0.150 & 0.093 & 0.140 \\
				LR Top2Leaves & {\bf -0.489} & 0.150 & 0.103 & 0.139 \\
				LR Top8Leaves & {\bf -0.397} & 0.147 & 0.120 & 0.137 \\
				LR Top22Leaves & {\bf -0.385} & 0.146 & 0.124 & 0.137 \\
				LR Top64Leaves & {\bf -0.384} & 0.146 & 0.124 & 0.137 \\
				\midrule
				LI SinglePoint & {\bf -0.652} & 0.015 & -0.052 & 0.005 \\
				LI Top1Leaves & {\bf -0.642} & 0.014 & -0.043 & 0.004 \\
				LI Top2Leaves & {\bf -0.616} & 0.014 & -0.033 & 0.003 \\
				LI Top8Leaves & {\bf -0.524} & 0.011 & -0.015 & 0.001 \\
				LI Top22Leaves & {\bf -0.512} & 0.011 & -0.012 & 0.001 \\
				LI Top64Leaves & {\bf -0.511} & 0.010 & -0.011 & 0.001 \\
				\bottomrule
			\end{tabulary}
		\end{small}
		\vskip -0.5cm
	\end{table}
	As expected, with all methods, the samples of the same type, as the filtered samples, are consistently the most influential. Indeed, removal of these samples increases the most the loss on $\bX_{test}$, and the derivative of the loss with respect to the weights of these samples is negative indicating that \emph{FastLeafInfluence} favors upweighting them.
	In all cases removal of elements with $y=0$ and $\mathrm{age}\in[40;50)$ is estimated to be profitable, also confirming the initial expectations.
	These results allow us to answer \textbf{RQ5} in the positive.
	
	%	In practice one could use the outputs of our influence estimation methods to determine biases in $\bX^{train}$, and adjust its distribution for domain adaptation/transfer learning.

	%\input{04-related.tex}
	% !TEX root = ./main.tex

\section{Conclusion}
In this work, we addressed the problem of finding train objects that exerted the largest influence on the GBDT's prediction on a particular test object. Building on the Influence Function framework for parametric models, we derived~\leafrefiti~and~\leafgradi, methods for estimating influences based on their respective proxy metrics, Proxies~\ref{proxy:loo}~and~\ref{proxy:gradients}. By utilizing the structure of tree ensembles, we also derived computationally efficient approximations to these methods, \fleafrefiti~and~\fleafgradi. In our experiments, through considering several applied scenarios, we showed the practical applicability of these approaches, as well as their ability to produce actionable insights allowing to improve the existing model.
	
	\clearpage
	\section*{Acknowledgments}
	We would like to thank Anna Veronika Dorogush for valuable commentary and discussions, as well as technical assistance with the CatBoost library.
	\bibliographystyle{icml2018}
	\bibliography{bibliography}
	
	%%%%%%%%%% Merge with supplemental materials %%%%%%%%%%

	%%%%%%%%%% Prefix a "S" to all equations, figures, tables and reset the counter %%%%%%%%%%
	\setcounter{equation}{0}
	\setcounter{figure}{0}
	\setcounter{table}{0}
	\setcounter{page}{1}
	\makeatletter
	\renewcommand{\theequation}{S\arabic{equation}}
	\renewcommand{\thefigure}{S\arabic{figure}}
	\renewcommand{\bibnumfmt}[1]{[S#1]}
	\renewcommand{\citenumfont}[1]{S#1}
	%%%%%%%%%% Prefix a "S" to all equations, figures, tables and reset the counter %%%%%%%%%%
	
	\clearpage
	\appendix
	\section*{Supplementary 1: Dataset Statistics and CatBoost Parameters}
	In this supplementary material, we specify the main statistics of the evaluation datasets (Table~\ref{table:datasets}) and values of CatBoost parameters used in the experiments (Table~\ref{table:parameters}):
	\begin{table}[h!]
		\caption{Dataset statistics.}
		\label{table:datasets}
		\centering
		\begin{small}
			\begin{tabulary}{\linewidth}{Lrrr}
				\toprule
				Dataset & No. Features & Train size & Test size \\
				\midrule
				Adult & 14 & 32,561 & 16,281 \\
				Amazon & 9 & 26,215 & 6,554 \\
				Upselling & 214 & 35,000 & 15,000 \\
				Hospital & 127 & 20,000 & 81,766 \\
				\bottomrule
			\end{tabulary}
		\end{small}
	\vskip -0.7cm
	\end{table}
	\noindent
	\begin{table}[h!]
		\caption{Dataset CatBoost parameters.}
		\label{table:parameters}
		\centering
		\begin{small}
			\begin{tabulary}{\linewidth}{LcccL}
				\toprule
				Dataset & No. Trees & Depth & Learn rate & Formula \\
				\midrule
				Adult & 100 & 6 & 0.2\phantom{0} & Newton \\
				Amazon & 100 & 6 & 0.15 & Newton \\
				Upselling & 100 & 6 & 0.15 & Newton \\
				\bottomrule
			\end{tabulary}
		\end{small}
	\end{table}
	
\end{document}